\documentclass{article}
\usepackage{ijcai20}

\usepackage{times}
\usepackage{amsfonts}

\usepackage{soul}
\usepackage{url}
\usepackage[hidelinks]{hyperref}
\usepackage[utf8]{inputenc}
\usepackage[small]{caption}
\usepackage{graphicx}
\usepackage{amsmath}
\usepackage{booktabs}
\usepackage{times}
\usepackage{graphicx}
\usepackage{latexsym}

\usepackage[square,numbers]{natbib}
\usepackage{amsthm}
\usepackage{placeins}
\usepackage{lineno}
\usepackage{mathtools}

\newtheorem{remark}{Remark}
\usepackage{amsmath}

\newtheorem{lemma}{Lemma}
\newtheorem{theorem}{Theorem}
\usepackage{mathrsfs}
\usepackage{graphicx,epsfig} 
\usepackage{longtable}
\usepackage[usenames]{color}
\usepackage{subfigure}

\def\rb0{r_0^{\flat}}

\def\ND{\mathcal{N}}

\newcommand{\bb}[1]{\boldsymbol{#1}}

\renewcommand{\hat}[1]{\widehat{#1}}

\renewcommand{\Gamma}{\varGamma}
\renewcommand{\Pi}{\varPi}
\renewcommand{\Sigma}{\varSigma}
\renewcommand{\Delta}{\varDelta}
\renewcommand{\Lambda}{\varLambda}
\renewcommand{\Psi}{\varPsi}
\renewcommand{\Phi}{\varPhi}
\renewcommand{\Theta}{\varTheta}
\renewcommand{\Omega}{\varOmega}
\renewcommand{\Xi}{\varXi}
\renewcommand{\Upsilon}{\varUpsilon}

\def\Var{\operatorname{Var}}

\def\R{I\!\!R}
\def\E{I\!\!E}
\def\P{I\!\!P}

\def\kappa{\varkappa}

\def\uv{\bb{u}}
\def\vv{\bb{v}}

\def\dLb12{T_h^{\flat}(\theta_1^{\flat}, \theta_2^{\flat})}

\def\alphab{\alpha^{\flat}}

\def\alphab12{\alpha^{\flat}(\theta, \theta_0)}
\def\chib12{\chi^{\flat}(\theta, \theta_0)}

\def\Lb0{L^{\flat}(\theta_0)}

\def\L0{L(\theta_0)}

\newcommand{\vertiii}[1]{{\left\vert\kern-0.25ex\left\vert\kern-0.25ex\left\vert #1 
    \right\vert\kern-0.25ex\right\vert\kern-0.25ex\right\vert}}

\def\vv{\bb{v}}

\usepackage[math]{easyeqn}
\begin{document}
%%%%%%%%%%%%%%

\title{Unsupervised non-parametric change point detection in quasi-periodic signals}

\author{Nikolay Shvetsov \and Nazar Buzun \and Dmitry V. Dylov \\ \affiliations
  Skolkovo Institute of Science and Technology\\Bolshoy blvd. 30/1, Moscow, 121205, Russia\\ \{nikolay.shvetsov, n.buzun, d.dylov\}@skoltech.ru} 

% \author{Anonymized for review} 

\maketitle

\begin{abstract}
We propose a new unsupervised and non-parametric method to detect change points in intricate quasi-periodic signals. 
The detection relies on optimal transport theory combined with topological analysis and the bootstrap procedure.
The algorithm is designed to detect changes in virtually any harmonic or a partially harmonic signal and is verified on three different sources of physiological data streams. We successfully find abnormal or irregular cardiac cycles in the waveforms for the six of the most frequent types of clinical arrhythmias using a single algorithm. The validation and the efficiency of the method are shown both on synthetic and on real time series. 
Our unsupervised approach reaches the level of performance of the supervised state-of-the-art techniques. 
We provide conceptual justification for the efficiency of the method and prove the convergence of the bootstrap procedure theoretically.
\end{abstract}

\section{Introduction}

New analytical approaches to the quasi-periodic signals with irregular rhythms~\cite{MAGLAVERAS1998191} -- such as those encountered in the electrocardiograms (ECG) -- are in major demand caused by the growth of the physiological monitoring market and by the inflating stock of consumer wearable solutions today. The abundance of unannotated time series data created by these modalities attracts notable theoretical effort in the search for the most robust change point detection (CPD) method, capable of operating in a fast, a model-agnostic, and an unsupervised manner. 

Indispensable with cardiovascular diseases (the top cause of mortality and the major life threat in adults worldwide), ECG has become the most frequently used clinical modality, attracting a multidisciplinary effort to detect conceivable markers of the heart problems in the recordings of the electrical function of the heart. Typical ECG is a one-dimensional time series measurement, close to a periodic signal, with each period consisting of three main parts: P wave, QRS complex, and T wave. Both the shapes and the temporal distribution of the PQRST waves carry important clinical information.  Among many things, the ECG modality allows detecting disruptions in the cardiac rhythm, being the major proxy for the doctors to diagnose heart arrhythmias~\cite{Guidi}. 

Mathematically, cardiac arrhythmias correspond to some degree of broken periodicity in the ECG data stream, having much in common with the many other quasi-periodic signals met in the nature~\cite{MAGLAVERAS1998191}. The six of the most frequent types of clinical arrhythmias are atrial flutter, atrial fibrillation, supraventricular tachycardia, premature atrial contraction, and ventricular rhythms~\cite{UNC}. Each of these conditions is defined by a different set of morphologic and temporal characteristics in the PQRST complex in the ECG signals.
 
In this work, we were motivated to develop a single, model-agnostic, unsupervised algorithm to detect all of such arrhythmias in a binary classification scenario, focusing on high detection specificity. Thanks to the non-parametric construction of the proposed change point statistic, we ensure zero modeling bias and applicability to a wide range of the incoming quasi-periodic signals, including the physiological ones. Besides ECG, we demonstrate the efficient application of the proposed algorithm to find abnormal rhythms in neuronal spiking streams and in periodic limb tremor data in patients with Parkinson's disease.

The formal problem statement is the following. Let $X_t$ be the quasi-periodic signal with a period $T$.  One has to test the hypotheses 
\begin{eqnarray}\begin{gathered}
\mathcal{H}_0:\{X_t \sim \P_{f_0(t/T)}, \; \forall t \in [0, n]\} \\ 
\mathcal{H}_1: \{\exists \tau^*: X_t \sim \P_{f_0(t/T)} \; ; X_t \sim \P_{f_1(t)} \;  \}\\ t \in [0, \tau^*] \quad \text{  and  } \quad t \in [\tau^*, n]
\end{gathered}\end{eqnarray}
In the notation above $\P$ represents a probability distribution, $n$ is the dataset size, $\tau^*$ is the change point time, $f_0(t/T)$ and $f_1(t)$ are the functions parametrizing the distributions.

% \subsubsection{Literature review}
Many groups have considered the problem of arrhythmia detection in the frameworks of modern machine learning methods~\cite{MAGLAVERAS1998191,Guidi}. Such approaches as Decision Trees, Random Forests, Support Vector Machines (SVM), Naive Bayes, and the Convolutional Neural Networks (CNN) were shown to be efficient for different supervised tasks with the corresponding pros and cons~\cite{Guidi}. Jun et.al. ~\cite{Tae} proposed methods of arrhythmia detection with the neural networks, allowing to achieve good values of recall. Further development included methods based on Genetic Algorithms~\cite{Finlay} and logistic regression ~\cite{Kawazoe}, both of which are now very popular for building the supervised classification models.

A limited number of unsupervised ECG analysis works have also appeared. In \cite{unsuper}, the authors describe unsupervised methods for feature extraction and clustering to prevent the false alarm in arrhythmia  detection. In \cite{unsupervised}, the authors describe application of k-means algorithms and a neural network to the ECG signal analysis. The application of topological data analysis (TDA) and Wasserstein metrics~\cite{Sommerfeld} for periodical signals was discussed in works by Perea~\cite{Perea,Perea1,Perea2}. Some selective  algorithms for the offline detection of multiple change points in multivariate time series are presented in \cite{nonparamreview,nonpar} and the CPD via Gaussian processes in~\cite{avanesov2019nonparametric}. The work by Buzun and Avanesov~\cite{Buzun} describes bootstrap application for CPD for the time series. Techniques of Gaussian approximation for the OT task are described in different works by Buzun ~\cite{Buzun1} and Chernozhukov ~\cite{Chernozhukov,Chernozhukov1,Chernozhukov2}.

The major difficulty in the statistical study of the problem (1) is twofold: the dependent data and the lack of a suitable parametric model for an intricate signal, such as ECG. To address these challenges, we propose a new pipeline shown in Figure \ref{fig:pipeline}. In the proposed algorithm, we resort to the optimal transport (OT) approach that is capable of building a non-parametric change point statistic to test the hypotheses. We propose to apply the TDA/OT approach not to the original signal, but to a projection of the quasi-periodic function into a closed curves space (the point cloud), allowing both the periodic and the morphologic components of the original signal's waveform to be considered. Eventually, we estimate quantiles of the change point statistic with the bootstrap procedure in order to set a threshold under the null hypotheses assumption. In the theoretical section, we prove a theorem about the convergence of the bootstrap distribution of the statistic to the real distribution, setting a foundation stone for a plethora of possible future works on TDA/OT analysis on periodic signals.

\begin{figure}[h]
     \centering
     \includegraphics[width=\columnwidth]{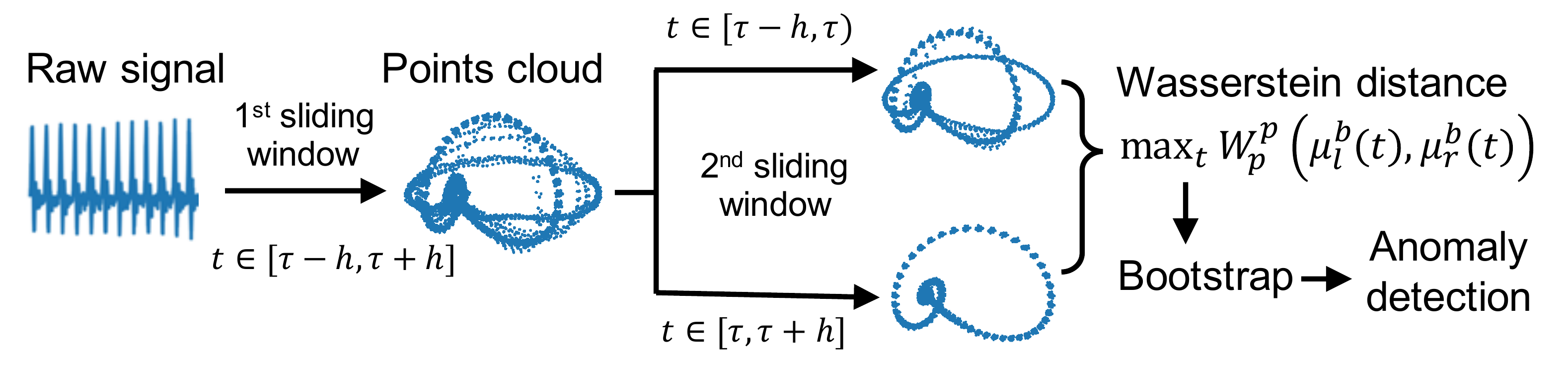} 
     \caption{Pipeline of the proposed algorithm, where $\tau$ -- the second sliding window center, $2h$ -- the second sliding window size, $W_p^p$ -- Wasserstein distance, $\mu_l^b(t)$, $\mu_r^b(t)$ -- Bootstrap measures in the left and the right parts of the second sliding window.}
     \label{fig:pipeline}
     \end{figure}
\FloatBarrier 

% In this work we apply methods of  and the optimal transport (OT) to derive a non-parametric CPD algorithm.  
% Perea et.al. represented a gene-encoded periodical function with a help of SW1PerS point cloud.
% The Wasserstein distance computation and bootstrap computation are described in the paper by Sommerfeld and Munk~\cite{Sommerfeld}, where the authors showed directional bootstrap for the Wasserstein distance. The Wasserstein distance was computed using the Sinkhorn algorithm~\cite{sinkhorn} which relies on the regularization parameter $\lambda$. 
% %The accuracy of the approximation is parameterized by a regularization parameter $\lambda$. 
% Computing this regularized OT problem results in two quantities: an upper bound on the actual OT distance, which called the dual-Sinkhorn divergence, as well as a lower bound, which can be used for nearest neighbor search under the OT metric. 

\section{Methodology}
\subsection{The first window: calculate point clouds}
The first step of our approach is to map the original time series into the point cloud. We use the method based on the sliding windows with 1-dimensional persistence scoring described in \cite{Perea}.
The main idea is to present the original periodical signal as a closed curve, which will help us to apply the optimal transport formalism to the quasi-periodic data.
Define
\begin{equation}
SW(t)=\left[X_t, \\ X_{t+s}, \\...\\, X_{t+Ms}\right]
\end{equation}
The sliding window ($SW$) makes an embedding of the signal $X_t$  at point $t$ into $\R^{M+1}$.
Iterating through different values of $t$ with a step $\Delta t$ one gets a collection of points called  sliding window point cloud of  $X_t$ (Figure \ref{fig:wass_boot_108}). A critical parameter for this embedding is the first window-size ($M s$). It is chosen to be equal to the duration of a single period in the signal (e.g., one PQRST cycle in the normal heart beat pattern).
In the next step we apply Principal Component Analysis (PCA) for this point cloud in order to increase robustness and have a possibility to visualize the rhythm disturbance. Two examples of the point clouds are shown in Figure \ref{fig:wass_boot_108}.

\begin{figure}[h]
     \centering     \includegraphics[width=\columnwidth, height=5cm]{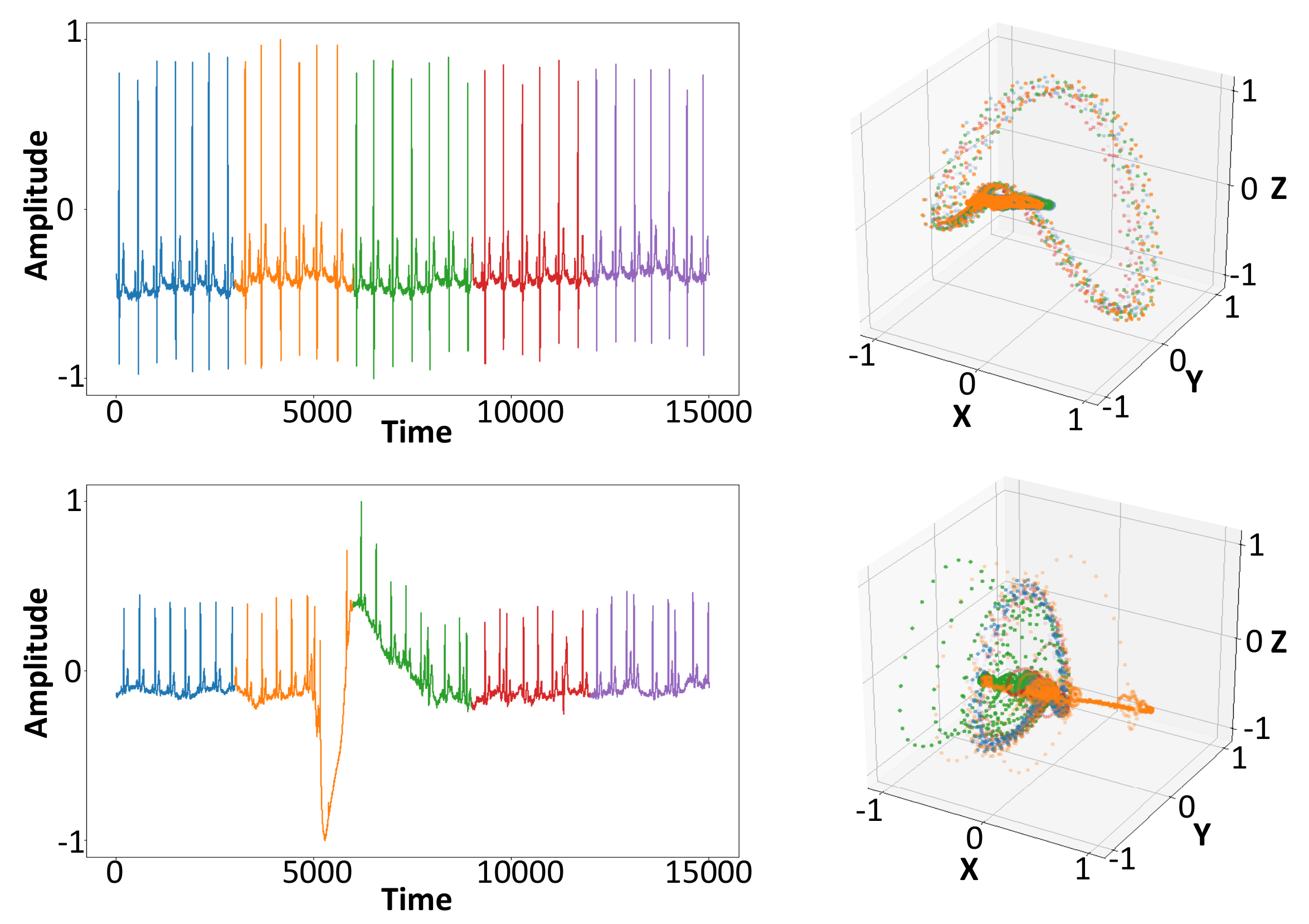}
     \caption{Point clouds of normal heart rhythm (top) and Atrial Flutter (bottom), with $\{\text{x,y,z}\}$ corresponding to the top tree PCA components. Colors help to visualize irregularities in the cloud.}
     \label{fig:wass_boot_108}
     \end{figure}
\FloatBarrier

\subsection{The second window: get Wasserstein distances} 
In order to find structural changes in the point cloud corresponding to the structural changes in the original time series we elaborate the method described in \cite{Buzun}. The main idea is that at each time step the procedure extracts a data slice from the point cloud, splits it in two equal-size parts, and computes Wasserstein distance between them. The size of the sliding window could be equal to several curve loops. The method avoids the rise of values of the Wasserstein distances due to fluctuations and neglects noise-driven changes in the curves, effectively tracing only the meaningful structural changes in the signal.

Wasserstein distance is defined on probability distribution pairs on some metric space~\cite{Sommerfeld}. By definition, the Wasserstein distance of degree $p$ between the probability measures $\mu$ and $\nu$ is
\begin{equation}
\begin{gathered}
W^p_p(\mu,v)=\left( \inf_{\gamma \in \Pi(\mu,v)} \int_{M*M} \| x -y \|^p d\gamma(x,y)\right) ^{1/p}
\end{gathered}\end{equation}
Lets introduce a change point statistic (the maximum distance over the window positions): 
\begin{eqnarray}
 T(2h)=\max_{\tau}W_p^p\big(\mu_l(\tau), \mu_r(\tau)\big) \qquad\\
\mu_l(\tau)=\frac{1}{h}\sum_{t=\tau-h}^{\tau-1}\delta _{X_{t}},\quad\mu_r(\tau)=\frac{1}{h}\sum_{i=\tau}^{\tau+h-1}\delta_{X_{t}} 
\end{eqnarray}

where $\delta_{X_t}$ is the Dirac function at position $X_t$ (i.e., a unit mass concentrated at location $X_t$), $\tau$ is the central point of the sliding window of length $2h$, implying that the data series within the sliding window is $(X_{\tau-h},\ldots,X_{\tau+h-1})$.

We calculate the Wasserstein distance for each position of the sliding window and create a new time series to be used for showing how the curves differ inside of the window. In practice, one can calculate Wasserstein distances via the Sinkhorn algorithm using the Optimal Transport Library~\cite{sinkhorn}.

\subsection{Moving blocks bootstrap for rhythm analysis}
%compute Wasserstein distance as follows is analogy with for all data, 
In this step, we compute Wasserstein distance and execute the bootstrap procedure:
\begin{eqnarray}
 T^b(2h)=\max_{\tau}W_p^p\big(\mu_l^b(\tau), \mu_r^b(\tau)\big) \qquad\\
\mu_l^b(\tau)=\frac{1}{h}\sum_{t=\tau-h}^{\tau-1}\delta _{X_{k(t)}},\quad\mu_r^b(\tau)=\frac{1}{h}\sum_{i=\tau}^{\tau+h-1}\delta_{X_{k(t)}} 
\end{eqnarray}
where the set $k(t)$ is generated by the Moving Block Bootstrap (MBB), and where the data is split and shuffled into $n$ blocks randomly. Naturally, we assume that the points located in the peaks of the plot of the Wasserstein distances correspond to the arrhythmia points on the original periodic signal. 
MBB was formulated in separate works by Künsch~\cite{Kunsch} and Lahiri~\cite{Lahiri} as new scheme to create pseudo-samples. The usual bootstrap forms new samples taking only random observations from the initial sample, whereas, the MBB performs this procedure only within a row of the formed blocks. We use a weighted block structure of the MBB, which generates random weights for each block and, importantly, preserves the structure of the original time series.

After the MBB resampling,  we create a list of change point statistic values ($T^b(2h)$)  and set the threshold with $\alpha$ confidence level corresponding to the border between the normal points and the points of arrhythmia (see Figure \ref{fig:em_ts_3}). It is assumed that quantiles of $T^b(2h)$ are close to the quantiles of $T(2h)$ (bootstrap consistency), which we justify in the theoretical part below.    

  \begin{figure}[h]
     \centering
     \includegraphics[width=\columnwidth]{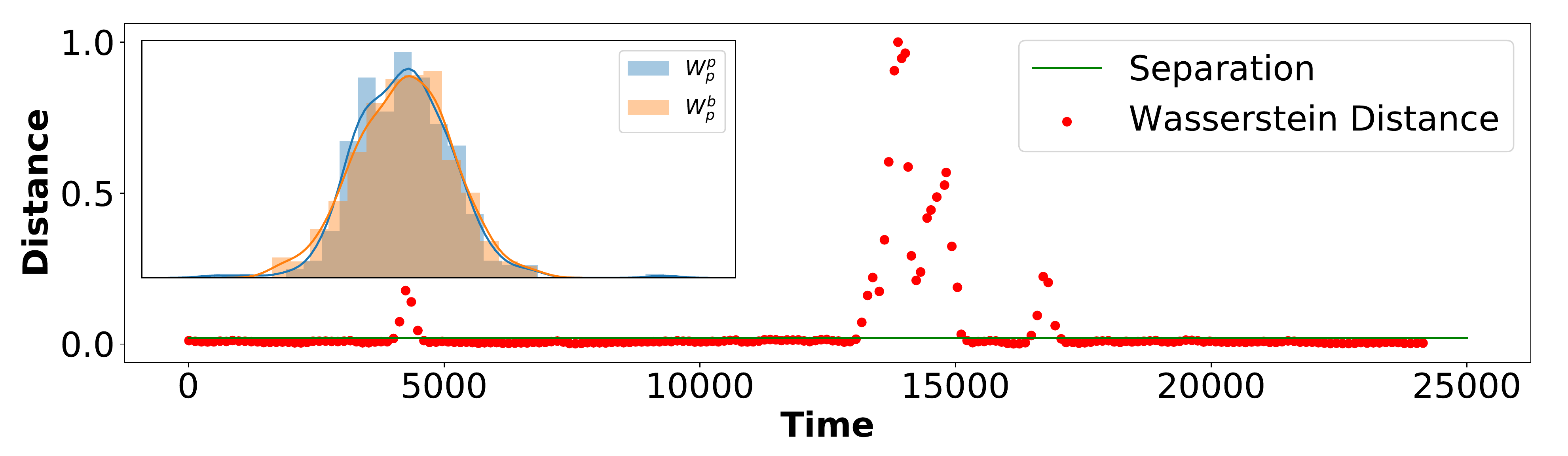}
     \caption{Values of Wasserstein distances (bootstrap) help detect abnormal rhythm in unsupervised manner. 
     %The original ECG signal shows predicted and annotated arrhythmia.
     }
     \label{fig:em_ts_3}
     \end{figure}
     
\subsection{Gaussian approximation}

Consider two point clouds of size $h$. They may belong to the same distribution (null hypothesis) or to different distributions. Assume that samples in each point cloud are independent (when we use block-bootstrap we may assume that blocks are independent). Bootstrap consistency requires that the distribution of the Wasserstein distances between these point clouds should be close to the distribution of the resampled ones. One necessary technique in the Bootstrap consistency proof is the Gaussian approximation. It appears that the Wasserstein distance between two point clouds under the null hypothesis can be approximated by the maximum of some Gaussian vector. In \cite{Sommerfeld}, this approximation is proved for the case of discrete distributions as a limit theorem.
\begin{theorem} 
Let measures $\boldsymbol{r}$, $\boldsymbol{s}$ are defined on a discrete set $\mathcal{X} = \{x_1,\ldots, x_N\}$ and  i.i.d. samples $X_{1}, \ldots, X_{h} \sim \boldsymbol{r}$ and $Y_{1}, \ldots, Y_{h} \sim \boldsymbol{s}$. Define convex sets:

\begin{eqnarray}
\Phi_{p}= \left\{\boldsymbol{u} \in \mathbb{R}^{N}: u_{x}-u_{x^{\prime}}\leq d^{p}\left(x, x^{\prime}\right),\quad x, x^{\prime}\in\mathcal{X}\right\}\\
\resizebox{1\hsize}{!}{$
\Phi_{p}^{*}(\boldsymbol{r}, \boldsymbol{s})=\left\{(\boldsymbol{u}, \boldsymbol{v}) \in \mathbb{R}^{\boldsymbol{X}} \times \mathbb{R}^{\boldsymbol{X}}: \begin{array}{l}{\langle\boldsymbol{u}, \boldsymbol{r}\rangle+\langle\boldsymbol{v}, \boldsymbol{s}\rangle= W_{p}^{p}(\boldsymbol{r}, \boldsymbol{s})} \\ {u_{x}+v_{x^{\prime}} \leq d^{p}\left(x, x^{\prime}\right), x, x^{\prime} \in \mathcal{X}}\end{array}\right\}$}
\end{eqnarray}

\noindent Multinominal covariance matrix $\Sigma(\boldsymbol{r})$ is

\begin{equation}
\begin{gathered}
\left[\begin{array}{cccc}{r_{x_{1}}\left(1-r_{x_{1}}\right)} & {-r_{x_{1}} r_{x_{2}}} & {\cdots} & {-r_{x_{1}} r_{x_{N}}} \\ {-r_{x_{2}} r_{x_{1}}} & {r_{x_{2}}\left(1-r_{x_{2}}\right)} & {\dots} & {-r_{x_{2}} r_{x_{N}}} \\ {\vdots} & {} & {\ddots} & {\vdots} \\ {-r_{x_{N}} r_{x_{1}}} & {-r_{x_{N}} r_{x_{2}}} & {\cdots} & {r_{x_{N}}\left(1-r_{x_{N}}\right)}\end{array}\right]\end{gathered}\end{equation}
such that with Gaussian random vector $Z \sim \ND(0, \Sigma(\boldsymbol{r}))$  it holds for empirical measures $\widehat{\boldsymbol{r}}_h$ and~$\widehat{\boldsymbol{s}}_h$:

\noindent\textbf{1) One sample - Null hypothesis} 
\begin{equation}
n^{\frac{1}{2 p}} W_{p}\left(\hat{\boldsymbol{r}}_{h}, \boldsymbol{r}\right)       \overset{d}{\longrightarrow} \left\{\max _{\boldsymbol{u} \in \Phi_{p}} \boldsymbol{u}^T  Z \right\}^{\frac{1}{p}}
\end{equation}

\noindent\textbf{2) One sample - Alternative} 

\begin{equation}
n^{\frac{1}{2}}\left(W_{p}\left(\hat{\boldsymbol{r}}_{h}, \boldsymbol{s}\right)-W_{p}(\boldsymbol{r}, \boldsymbol{s})\right) 
\overset{d}{\longrightarrow}\end{equation}$$\frac{1}{p} W_{p}^{1-p}(\boldsymbol{r}, \boldsymbol{s})\left\{\max _{(\boldsymbol{u}, \boldsymbol{v}) \in \Phi_{p}^{*}(\boldsymbol{r}, \boldsymbol{s})} \boldsymbol{u}^T Z \right\}$$

\noindent\textbf{3) Two samples - Null hypothesis.} If \textbf{r}=\textbf{s} and $h$ is approaching infinity such that $h \rightarrow \infty$ and $h \rightarrow \lambda \in(0,1)$, then:
\begin{equation}
\left( \frac{n m}{n+m} \right )^{\frac{1}{2 p}} W_{p}\left(\hat{\boldsymbol{r}}_{h}, \hat{\boldsymbol{s}}_{h}\right) 
\overset{d}{\longrightarrow} 
\left\{\max _{\boldsymbol{u} \in \Phi_{p}} \boldsymbol{u}^T  Z  \right\}^{\frac{1}{p}}
\end{equation}

 \noindent\textbf{4) Two samples - Alternative} With n and m approaching infinity such that $h \rightarrow \infty$ and $h \rightarrow \lambda \in[0,1]$, then:
\begin{equation}
\left( \frac{h^2}{2h} \right )^{\frac{1}{2 p}} \left(W_{p}\left(\hat{\boldsymbol{r}}_{h}, \hat{\boldsymbol{s}}_{h}\right)-W_{p}(\boldsymbol{r}, \boldsymbol{s})\right) \overset{d}{\longrightarrow}
\end{equation}
$$
\frac{1}{p} W_{p}^{1-p}(\boldsymbol{r}, \boldsymbol{s})\left\{\max _{(\boldsymbol{u}, \boldsymbol{v}) \in \Phi_{p}^{*}(\boldsymbol{r}, \boldsymbol{s})} \sqrt{\lambda} \boldsymbol{u}^T  Z  +\sqrt{1-\lambda} \boldsymbol{v}^T  Z_s \right\}
$$

\end{theorem}

Below, we will extend the result from \cite{Sommerfeld} for continuous case and estimate the resulting convergence rate. We will need two lemmas for that. 

\begin{lemma}\label{max_gar} \cite{Chernozhukov3}\cite{Buzun2}
Let independent samples $X_{1}, \ldots, X_{h} \in \R^{p}$ be centered random vectors. Their Gaussian counterparts are $Y_{i} \sim \ND \left(0, \Var\left[X_{i}\right]\right)$.  Denote their sum by 
\[
S^X =  \frac{1}{\sqrt{h}} \sum_{i=1}^h X_i, 
\quad
S^Y = \frac{1}{\sqrt{h}} \sum_{i=1}^h Y_i
\]
Assume $\exists b>0$ such that for all $ j \in 1 \ldots p$:   $\E (S^X_j)^2 \geq b$  and   $\exists G_{h} \geq 1$ such that for $k \in\{1,2\}$
\begin{equation}\begin{gathered}
\frac{1}{h} \sum_{i=1}^{h} \E\left[\left|X_{i j}\right|^{2+k}\right] \leq G_{h}^{2+k} \end{gathered}\end{equation}
\begin{equation}\begin{gathered}
\E \left[\exp \left(\frac{\left|X_{i j}\right|}{G_{h}}\right)\right] \leq 2 
\end{gathered}\end{equation} Then for a set $A$ of hyper-rectangle form
\begin{eqnarray}
\sup _{A } \left|\P\left\{S^{X} \in A\right\}-\P\left\{S^{Y} \in A\right\}\right|\nonumber \leq O \left(\frac{G_{h}^{2} \log ^{7}(p h)}{h}\right)^{1 / 6}
\end{eqnarray}
\end{lemma}

\begin{lemma}[Anti-concentration] \label{aconc}\cite{Chernozhukov}
 Let $\mathcal{F}\subset \mathcal{L}^2(P)$ be a separable class of measurable functions and entropy of $\mathcal{F}$ be finite. Denote by $G(f)$, $f \in \mathcal{F}$ a Gaussian random process with zero mean and covariance depended on measure $P$:  
 \begin{equation}
  \E[G(f)G(g)] = \int f(x) g(x) d P(x)
 \end{equation}
 Suppose that there exist constants $\underline{\sigma}$, $\overline{\sigma}>0$ such that\\ $\underline{\sigma}^2\leq \E f^2 \leq\overline{\sigma^2}$ for all $f \in \mathcal{F}$. Then $\forall$  $x$ and $\Delta > 0$
 \begin{equation}
\P \left( \sup _{ f \in \mathcal{F} } G(f) \in  [x, x + \Delta] \right) \leq C_{A} \Delta
 \end{equation}
where
\begin{equation}
C_{A} = 
O \left(\E\left[\sup _{f \in \mathcal{F}} G(f)\right]+\sqrt{1 \vee \log (\underline{\sigma} / \Delta )}\right)
\end{equation}

\end{lemma}

\begin{proof}

  Basing on V. Chernozhukov's work \cite{Chernozhukov1}, \cite{Chernozhukov2} maximum of a Gaussian vector $Z$ has the following anti-concentration 
\begin{equation}
\begin{gathered}
\P \left(\max _{1 \leq j \leq p} Z_{j}  \in  [x, x + \Delta] \right)\\ \leq \Delta \, O \left( \E \max _{1 \leq j \leq p} Z_{j}  +\sqrt{1 \vee \log (\underline{\sigma} / \Delta)}\right)
\end{gathered}
\end{equation}
Make a finite $\epsilon$-net $\{ f_1, f_2, \ldots \}$ for $\mathcal{F}$ and set $Z_j$ equals to value of $G$ in the center of $j$-th cell $G(f_j)$,\\ such that $\forall$ $j$
and $\| f - f_j \|_P \leq \epsilon$, $\epsilon \to 0$
\begin{equation}
 \P ( | Z_j - G(f) |  > \delta ) \to 0
\end{equation}
and subsequently
\begin{equation}
 \P \left( \max_{j, \| f - f_j \|_P \leq \epsilon} | Z_j - G(f) |  > \delta \right) \to 0
\end{equation}
and
\begin{equation}
 \max_j Z_j \overset{Pr}{\longrightarrow} \max_f G(f) 
\end{equation}
Note that convergence by probability yields convergence by distribution, so
\begin{equation}
\P \left(\sup_f  G(f) \in  [x, x + \Delta] \right) \to 
\P \left(\max_j  Z_j \in  [x, x + \Delta] \right)
\end{equation}
and 
\begin{equation}
\E \sup _{f} G(f) \to \E \max_j  Z_j
\end{equation}

\end{proof}

\begin{remark}

The original proof one may find in Lemma A.1 from article \cite{Chernozhukov}.
We have used finite entropy assumption in the previous Lemma because it ensures the  existence of process $G(f)$ according to Dudley’s
criterion for sample continuity of Gaussian processes.

\end{remark}
\noindent Denote 
\[
\Phi_p = \{(u,v): | u(x) + v(y) | \leq \| x - y \|^p, \, \forall x, \forall y \in R^d  \}
\]
\[
\P(\varphi_1, \varphi_2) = \P \left( \sqrt{h} \max_{(\uv, \vv) \in \Phi_p} \langle \uv,  \phi_1 \rangle + \langle \vv, \phi_2 \rangle  > x \right)
\]\[
\P_\varepsilon(\varphi_1, \varphi_2) = \P \left( \sqrt{h} \max_{(\uv, \vv) \in \epsilon\text{-net}(\Phi_p)} \langle \uv,  \phi_1 \rangle + \langle \vv, \phi_2 \rangle  > x \right)
\]

\begin{theorem} 
\label{wasd_gar}
Consider i.i.d. samples $X = \{ X_{1}, \ldots, X_{h} \}$ and $Y = \{Y_{1}, \ldots, Y_{h} \}$ with a bounded support space $\Omega$ of dimension $d$.  Exist  Gaussian vectors $Z_1, Z_2 \in \ND(0, \Sigma_\psi) $ and generalized Fourier basis $\{\psi_i\}_{i=1}^{\infty}$, such that  
\begin{equation}
\Sigma_\psi =  \E \psi  \psi^T(X_1)
\end{equation}
and the Wasserstein distance between the samples can be approximated by the maximum of Gaussian process with the following upper bound
\begin{equation}\begin{gathered}
\left |
\P \bigg(  \sqrt{h} W^p_{p}(X, Y)  > x \bigg) - \P \left( Z_1^T \psi,  Z_2^T \psi \right)  \right | \\
\leq  C_A O \left(  \frac{\log h}{ h } \right)^{\frac{1 }{6 + 7d/p}},\end{gathered}
\end{equation}
where $C_A$ can be written as $$
%\begin{*eqnarray}
O \left(\E  \max_{(\uv, \vv) \in \Phi_p} \langle \uv,  Z_1^T \psi \rangle + \langle \vv,  Z_2^T \psi \rangle  + \sqrt{1 \vee \log (\underline{\sigma} / \mu_3 )}\right)$$  
%\begin{gathered}
% $\mu_3  =  \frac{1}{h^{3/2}}  \sum_{i=1}^h   \E \| X_i - Y_i  \|^{3p}.$
%\end{gathered}
%\end{*eqnarray}
\end{theorem}

\begin{remark}
From the practical sense, resampling of Wasserstein distance should entail data normalization in order to restrict $\Omega$ and should  keep the power $p$ close to the data dimension $d$.
\end{remark}

\begin{remark}
In combination with Gaussian comparison \cite{Chernozhukov3} one may show the bootstrap consistency, i.e. 
\[
\P \left( Z_1^T \psi,  Z_2^T \psi \right) \approx \P \left( (Z^b)_1^T \psi,  (Z^b)_2^T \psi \right)
\]
and subsequently 
\[
\P \bigg(  \sqrt{h} W^p_{p}(X, Y)  > x \bigg) \approx \P \left( (Z^b)_1^T \psi,  (Z^b)_2^T \psi \right)
\]
From this also follows that $T(2h)$ converges to $T^b(2h)$ by distribution when $h \to \infty$.
\end{remark}

\begin{proof}
The dual formulation of Wasserstein distance is 
\begin{equation}
W_p^p(X, Y) = \max_{(\uv,\vv) \in \Phi_p } \langle u, \phi_X \rangle + \langle v, \phi_Y \rangle\\
\end{equation}
\begin{equation}
\phi_X(x)=\frac{1}{h} \sum_{i=1}^h \delta(x - X_i),\,\,\phi_Y(x)=\frac{1}{h} \sum_{i=1}^h \delta(x - Y_i)
\end{equation}
\noindent Show how the covering number of $\Phi_p$ depends on the support space of empirical measures $\Omega$. Construct an $\varepsilon$-net on empirical measures. Its cardinality is $h^{N(\Omega, \varepsilon)}$ since each  $\varepsilon$-cell of $\Omega$ may contain from $0$ to $h$ points.  For each measures pair  $(\mu_1^{\varepsilon}, \mu_2^{\varepsilon} )$ from  $\varepsilon$-net one may set in correspondence pair $(u_\varepsilon,v_\varepsilon) \in \Phi_p$ such that $(u_\varepsilon,v_\varepsilon) $ is constant inside each cell of $\Omega$ and 
\begin{equation}
W^p_p(\mu_1^{\varepsilon}, \mu_2^{\varepsilon}) = \langle u_\varepsilon, \mu_1^{\varepsilon} \rangle + \langle v_\varepsilon, \mu_2^{\varepsilon} \rangle
\end{equation}
and subsequently for each arbitrary pair of empirical measures $(\mu_1, \mu_2)$ on $\Omega$ there is an element  $(u_\varepsilon,v_\varepsilon) \in \Phi_p$ with property 
\begin{eqnarray}
\langle u_\varepsilon, \mu_1 \rangle + \langle v_\varepsilon, \mu_2 \rangle = \langle u_\varepsilon, \mu_1^{\varepsilon} \rangle + \langle v_\varepsilon, \mu_2^{\varepsilon} \rangle\\
%\end{equation}
%\begin{equation}
%\begin{gathered}
W^p_p(\mu_1, \mu_2) -  \langle u_\varepsilon, \mu_1 \rangle + \langle v_\varepsilon, \mu_2 \rangle =\nonumber\\ W^p_p(\mu_1, \mu_2) - W^p_p(\mu_1^\varepsilon, \mu_2^\varepsilon) \leq 2\varepsilon^p
%\end{gathered}
\end{eqnarray}
Decompose  densities $\varphi_X$, $\varphi_Y$ in $\{\psi_i(x)\}$ basis 
\begin{equation}\begin{gathered}
\langle u, \phi_X \rangle + \langle v, \phi_Y \rangle =\\ \left \langle u, \left( \frac{1}{h} \sum_i \psi(X_i) \right)^T \psi \right\rangle + \left \langle v,  \left( \frac{1}{h} \sum_i \psi(Y_i) \right)^T \psi \right \rangle\end{gathered}
\end{equation}
In order to replace $\{\psi(X_i) \}$ and $\{\psi(Y_i) \}$ by Gaussian vectors  and use anti-concentration one has to make an $\varepsilon$-net approximation of $(u,v)$ functions. We have shown above that the cowering number  of $\Phi_p$  may by restricted by $O(h^{1/\varepsilon^{d/p}})$. So one may set 
\begin{equation}\log p_\varepsilon =  \frac{1}{\varepsilon^{d/p}} \log(h) + O(1)\end{equation}determining the dimension of maximum function. On  $\varepsilon$-net Lemma \ref{max_gar} gives upper bound  
\begin{equation}
\left | \P_\varepsilon(\varphi_X, \varphi_Y) - \P_\varepsilon(Z_1^T \psi,  Z_2^T \psi)  \right| \leq 
O \left(\frac{G_{h}^{2} \log ^{7}(p_\epsilon h)}{h}\right)^{1 / 6}
\end{equation}
where 
\[
G_{h}^{2 + k} =  \E (\uv(X_1) + \vv(Y_1))^{2+k} \leq \E \| X_1 -Y_1 \|^{p (2+k)}  
\]To make a step from $\P_\varepsilon$ to $\P$ remind that functions $u$ and $v$ are $\| \cdot \|^p$- Lipschitz and subsequently 
\[
\max_x | \uv(x) - \uv_{\epsilon}(x) | \leq \varepsilon
\]
and using Lipschitz property with Lemma \ref{aconc} one gets  
\[
| \P_\varepsilon(Z_1^T \psi,  Z_2^T \psi) - \P(Z_1^T \psi,  Z_2^T \psi)  | \leq O (C_A \varepsilon)
\]
\noindent Setting  optimal 
\[
\varepsilon = \left( \frac{1}{C_A h^{1/6}}  \right)^{\frac{1}{1 + 7d/6p}}
\]
gives the initial statement.\end{proof}

\section{Experiments}
In the experimental section we will demonstrate the use of the Theorem 1 on various physiological measurements and on their synthetic models. Of main interest to us is the ECG signal, but the data streams from a wearable sensor that had recorded limb tremor activity in a patient with Parkinson's disease will also be analysed.
\subsection{Real ECG data}
We used the MIT-BIH arrhythmia dataset from the PhysioNet~\cite{MIT}. The MIT-BIH Arrhythmia Dataset contains 48 half-hour excerpts of two-channel ambulatory ECG recordings, studied by the BIH Arrhythmia Laboratory between 1975 and 1979. 23 recordings were chosen at random from a set of 4000 24-hour ambulatory ECG recordings and include most common arrhythmia types. The remaining 25 recordings include less common but clinically significant arrhythmias. Each record contains two 30-min ECG lead signal (mostly MLII lead and lead V1/V2/V4/V5) sampling the data at a frequency of 360Hz. Our algorithm proved to work without any data pre-processing or noise reduction and detected all types of arrhythmia (see results in Table 1).

\subsection{Artificial ECG data}
To visualize the mechanism of arrhythmia detection, and to populate the arrhythmia classes equally, we also developed an auxiliary model to generate artificial ECG. This model can simulate the normal beat and produce different types of arrhythmia at random time moments: Atrial Flutter, Atrial Fibrillation, Supraventricular Tachycardia, Premature Atrial Contraction, and Ventricular Rhythms -- all according to the initialization parameters of the model. The generation is produced via the discrete wavelet transform in the form:
\begin{equation}
T_{m,n}=\int^{\infty}_{-\infty}x(t)\psi_{m,n}(t)dt
\end{equation}
where $\psi_{m,n}$ is the orthonormal wavelet basis. In this work we use Daubechies wavelet, generated with \emph{scipy} library, using hyperparameter different types of arrhythmia were created, also Gaussian noise $\ND(\mu, \sigma)$ was added to provide realistic ECG data.
%\begin{equation}
%p_{G}(z)={\frac  {1}{\sigma {\sqrt  {2\pi }}}}e^{{-{\frac  {(z-\mu )^{2}}{2\sigma ^{2}}}}}
%\end{equation}
\subsubsection{Artificial normal rhythm model}
To simulate normal heart beat we used its verbatim definition from the medical textbooks~\cite{UNC}. Normal sinus rhythm is a periodical signal, with the heart rate ranging from 60 to 100 bpm. The QRS complex is normal, the P wave always exists before the QRS, the T wave is visible after the QRS. Wavelets with Gaussian noise result in smooth curves in the point cloud, emphasizing the clean normal cycling beat trajectory (See Figure~\ref{fig:signals_curves}).

% \begin{figure}[h]
%  \centering
%  \includegraphics[width=\columnwidth]{normal.pdf}
%  \caption{Artificial normal heart beat and corresponding point cloud.}
%  \label{fig:em_ts11}
%  \end{figure}
%  \subsubsection{Artificial arrhythmia model}
 
  \begin{figure}[h]
 \centering
 \includegraphics[width=\columnwidth]{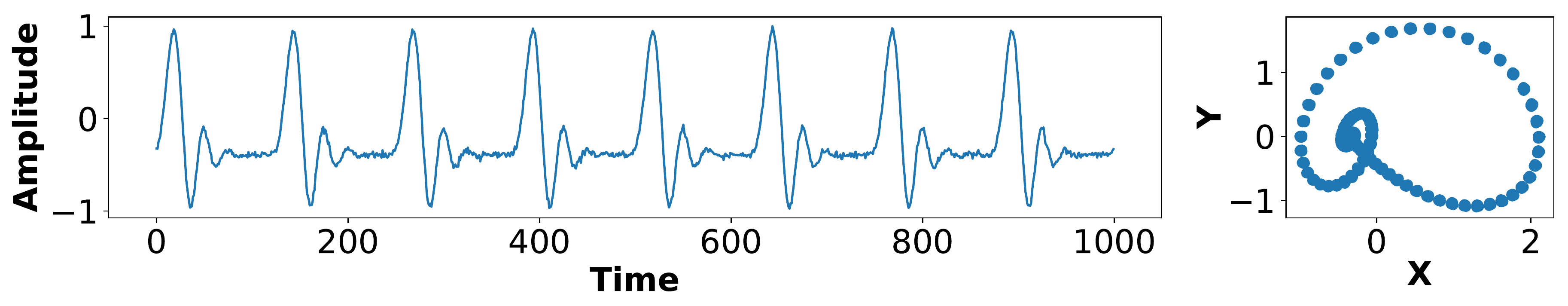}
 \includegraphics[width=\columnwidth]{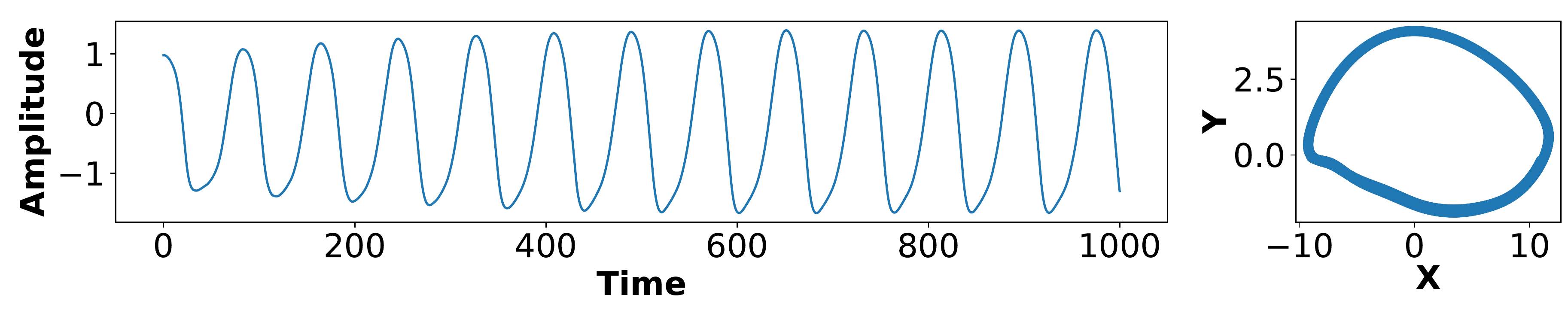}
 \caption{Artificial ECG (top) and spiking neuronal activity from~\cite{krylov2019RL} (bottom) with corresponding point clouds.}
 \label{fig:signals_curves}
 \end{figure}
 
 We consider 5 of the most frequent types of arrhythmia and 1 signal with an unknown random rhythm anomaly, each of them corresponding to some unique PQRST characteristics. Each type of arrhythmia was initiated at a random time moment within a given synthetic time series stream. This was done to understand the performance of the bootstrap detection on the ideal data, to test its stability to the noise (omitted for brevity), and to learn the changes that appear in the point cloud when particular features of a cardiac malady emerge in the signal.\footnote{To the authors' knowledge, this kind of ECG representation -- the point could that automatically boosts the visibility of an abnormal rhythm -- has never been suggested for the clinical use before. We speculate that it could be easily integrated into the physiologic systems to accompany or, perhaps, even to substitute the conventional ECG running monitors. With time, doctors can get accustomed to looking on the cyclic clouds just like they have gotten used to the waveforms of conventional ECG.}  Simulated waveforms of each cardiac arrhythmia (and the corresponding point clouds) were calculated using basic textbook in cardiology~\cite{UNC}; the results are presented in Figure \ref{fig:art_signals}.
 
 \begin{figure}[h]
 \centering
 \includegraphics[width=\columnwidth]{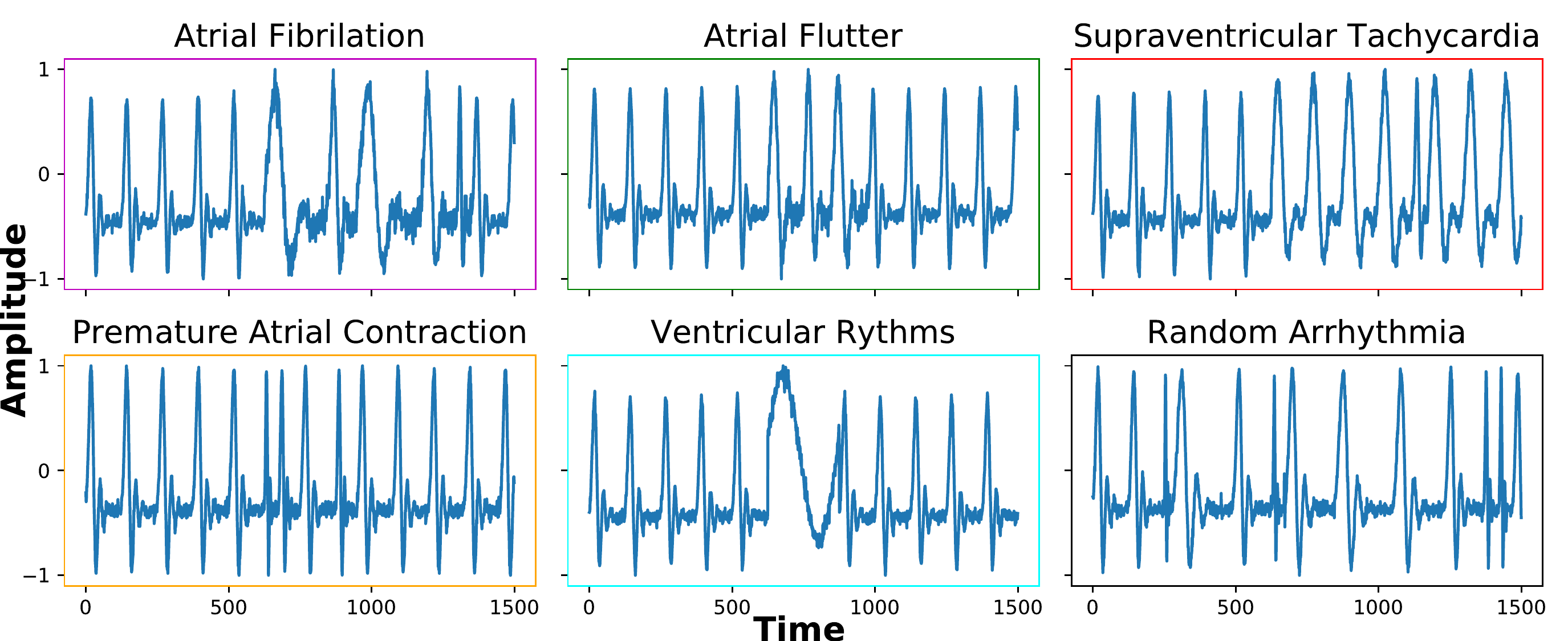}
 \includegraphics[width=\columnwidth]{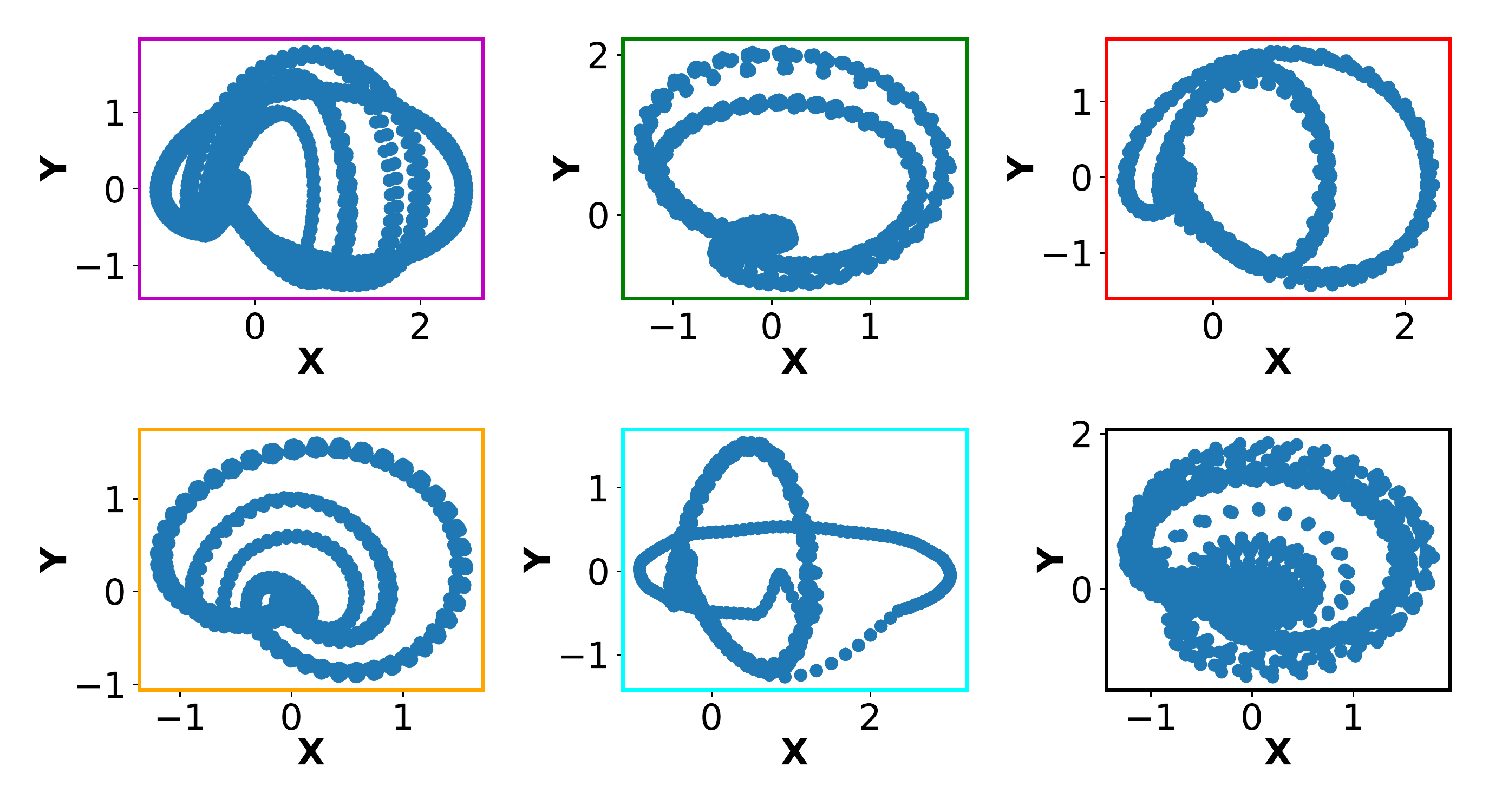}
 \caption{Artificial ECG with arrhythmia. The time series data and the corresponding point clouds are encoded with the same frame color.}
 \label{fig:art_signals}
 \end{figure} 
 
%  The artificial model was used to generate cardiac rhythm abnormalities with the following features. \emph{Atrial Fibrillation}, the most common arrhythmia, is characterized by irregular sinus beat without the normal P wave (sometimes it is replaced by the F wave) and with a non-discernible PR interval. The heart rate exceeds 100 bpm. \emph{Atrial Flutter} is the second most common type of arrhythmias. The normal P wave is also absent (sometimes it is replaced by the F wave), the PR interval is not measurable, the QRS complex is normal, the heart rate is fast. \emph{Supraventricular Tachycardia} is characterized by a normal rhythm, the P wave is usually hidden by the QRS complex and the PR interval is not discernible. \emph{Premature Atrial Contractions} is generally a normal beat that appears early, causing an irregular rhythm. P wave has an abnormal shape, with a short PR interval appearing, and the QRS complex looking normal. \emph{Ventricular Rhythm} is characterized by an irregular rhythm, no P waves, the PR interval immeasurable, and the QRS complex being very wide and bizarre~\cite{UNC}. 
% \FloatBarrier 

\subsection{Comparison with state-of-the-art}

Each ECG series was split to parts of different size (40,000, 80,000, and 120,000 points). If we take the indexes of the points, whose values are above the separation line calculated in the bootstrap procedure, these points in the original ECG will be the points with the arrhythmia. The PhysioNet dataset has the annotations accompanying the data; therefore, it is possible to compare the predicted labels of the points with the ground truth.

The parameters of the first sliding window have the following values $M s=450$, $s=1$, $\Delta t=2$ ($\Delta t$ is step of moving window), corresponding to the typical ECG sampling parameters, such as those in the MIT-BIH dataset. The size of the second sliding window is equal to $4$ curve loops, it  means that the window separates the series into $2$ parts with $2$ curve loops in each. We chose the confidence level $\alpha$=5\%.

To gauge the performance of the algorithm, we use sensitivity and specificity of the prediction~\cite{Guidi,Tae}. To calculate them we used a hold-out test set comprising the ECG signals with the normal heart beat (160 parts) and the ECG with arrhythmias (192 parts). As a result, the specificity of 86\%, and the sensitivity of 92\% were obtained. We have also calculated the same metrics for the artificial data, and for all types of arrhythmia (42 time series, with arrhythmia in different parts of series). The results are the following: sensitivity 97.2\% with 4.1\% standard deviation; specificity 96.2\% with 3.1\% standard deviation. Optimal choice of prediction threshold and the size of the sliding windows define the trade-off between the high recall and the low false positive rate. 

Comparison of our algorithm against several other approaches is shown in Table 1. We note that the pipeline in Figure 1 was meant to be as simple as possible, providing a robust statistical approach to predict abnormal rhythms in an unsupervised manner with high computational efficiency. Enhancing the pipeline by obvious combination with the deep learning or the hybrid model-based analysis methods is beyond the scope of this paper. Relevant to the clinical approbation, the method was tested (and correctly detected) on the short-episode arrhythmia in the long-term monitoring data stream (Figure 6).

% \begin{table}
% \centering 
% \caption{Results comparison with different methods}
% \begin{tabular}{ |p{1cm}|p{1cm}|p{1cm}|p{1cm} |p{1cm}| }
%  \hline
%  Classifier &Sens\%&Spec\% & Supervision\\
%  \hline
%  \textbf{Non Parametric method}	& \textbf{Proposed(on real data)}	&\textbf{92 $\pm$ 4}&\textbf{86 $\pm$ 6}&\textbf{Unsupervised}\\
 
%  & \textbf{on artificial data}	&\textbf{97.2 $\pm$ 4.1}&\textbf{96.2 $\pm$ 3.1}\\
%  \hline
%  SVM + PCA + Pan–Tompkins &Jing Hua, Hua Zhang \cite{Jing}& 70&98&Semi supervised\\
%  \hline
%  2D CNN&Tae Joon Jun\cite{Tae}	& 99.57	&97.85&Supervised\\
%  \hline
%  Echo State Network	& Miquel Alfares\cite{Finlay}	& 84.4&	99.7 & Supervised\\
%  \hline
%  LD  QRS-based and time interval features	&Chazal et. al\cite{Chazal}	& 75.9&	77.7 & Supervised\\
%  \hline
%  Logistic Regression,syncope&Kavazoe et al. \cite{Kawazoe}& 97&63&Supervised\\
%  \hline
%  Decision trees, heart rate+Legendre polynomial coefficients features&Faganely and Jager \cite{Faganeli}	& 98.1	&85.2&Supervised\\

%   \hline
%  \end{tabular}
%  \label{table:comparison}
%  \end{table}

\begin{table}
\begin{center}
{\caption{Comparison of proposed approach with state-of-the-art. Definitions of sensitivity and specificity follow those in Ref.~\cite{Tae}.}\label{table2}}
\begin{tabular}{lccccccc}
\hline
\rule{0pt}{12pt}
%\cline{2-8}
\rule{0pt}{12pt}
Method&Sens\%&Spec\%&Supervision
\\
\hline
\\[-6pt]
\emph{1}&\textbf{92.0 $\pm$ 4.0}&\textbf{86.0 $\pm$ 6.0}&\textbf{$\Diamond$}\\
\emph{1*}&\textbf{97.2 $\pm$ 4.1}&\textbf{96.2 $\pm$3.1}&\textbf{$\Diamond$}\\
\emph{2}~\cite{Truong2018SelectiveRO,killick2012optimal}&91.6&77.0&$\Diamond$\\
\emph{2*}~\cite{Truong2018SelectiveRO,killick2012optimal}&88.9&84.1&$\Diamond$\\
\emph{3}~\cite{Adams2007BayesianOC}&92.0&80.6&$\Diamond$\\
\emph{3*}~\cite{Adams2007BayesianOC}&85.8&88.9&$\Diamond$\\
\emph{4}~\cite{Jing}&70.0&98.0&$\bigtriangleup$\\
\emph{5}~\cite{Tae}&99.6&97.8&$\Box$\\
\emph{6}~\cite{Finlay}&84.4&	99.7&$\Box$\\
\emph{7}~\cite{Chazal}&75.9&77.7&$\Box$\\
\emph{8}~\cite{Kawazoe}&97.0&63.0&$\Box$\\
\emph{9}~\cite{Faganeli}&98.1&85.0&$\Box$\\
\\
\hline
\\[-6pt]
\multicolumn{8}{l}{$\Diamond$ Unsupervised\ \
$\bigtriangleup$ Semi-supervised\ \
$\Box$ Supervised}\\
\multicolumn{8}{l}{\emph{1}:~\textbf{Bootstrap}~on~real data,~~\emph{1*}:~\textbf{Bootstrap} on artificial data }\\
\multicolumn{8}{l}{\emph{2}: Ruptures(PELT) on Wasserstein distance data}\\
\multicolumn{8}{l}{\emph{2*}: Ruptures(PELT) on Euclidean distance data}\\
\multicolumn{8}{l}{\emph{3}: BOCP on Wasserstein distance data}\\
\multicolumn{8}{l}{\emph{3*}: BOCP on Euclidean distance data}\\
\multicolumn{8}{l}{\emph{4}: SVM + PCA~~\emph{5}:~2D CNN~~~\emph{6}:~Echo State Network }\\
\multicolumn{8}{l}{\emph{7}: LD QRS- and time interval-based features }\\
\multicolumn{8}{l}{\emph{8}: LR~~~~\emph{9}: DT+Heart rate features}\\ 

\end{tabular}
\end{center}
\end{table}

\begin{figure}[h]
 \centering
 \includegraphics[width=\columnwidth]{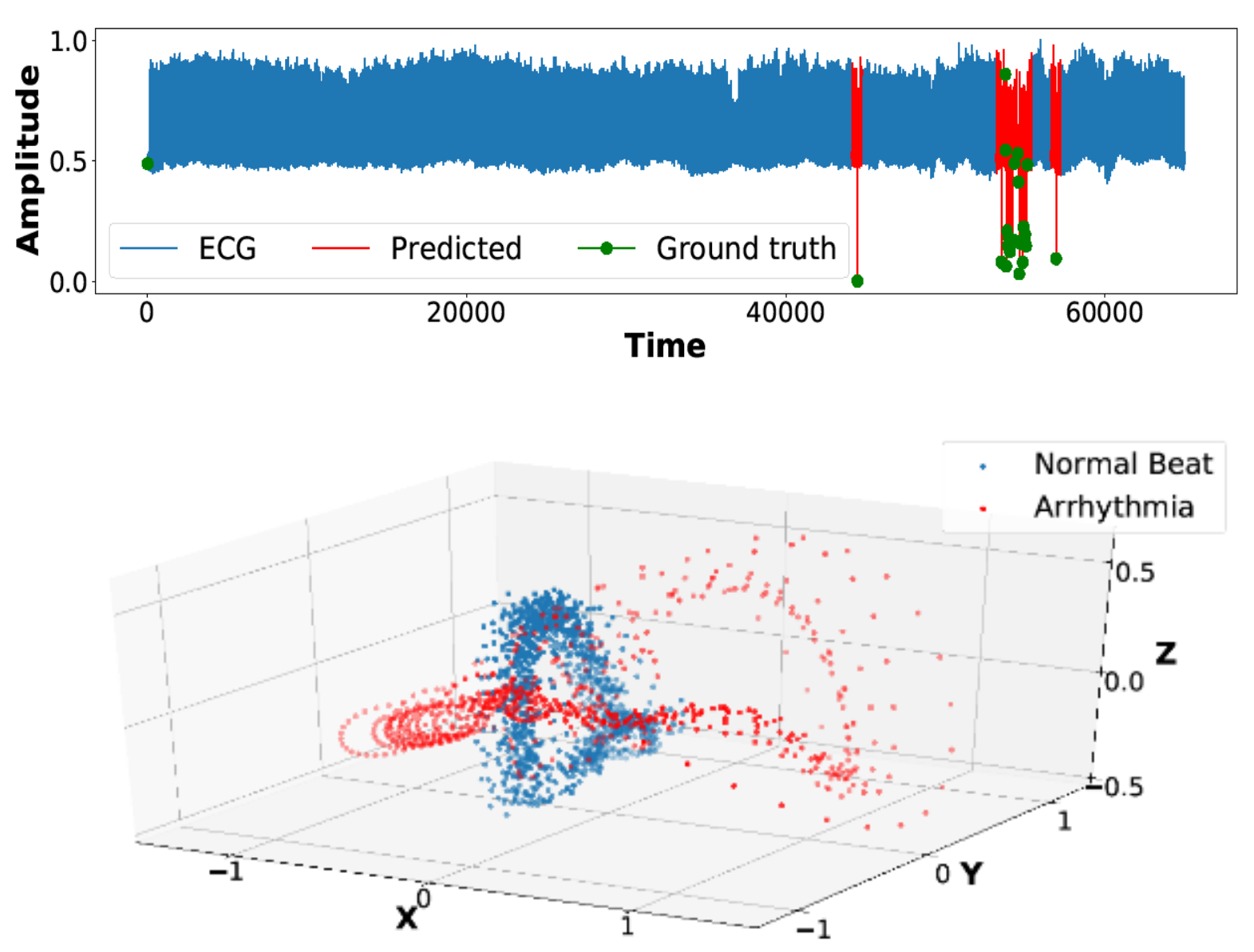}
 \caption{Algorithm's performance on a long-term monitoring data and the real ECG time series. The detected arrhythmia is visible both in the Wasserstein graph and in the 3D point cloud.}
 \label{fig:em_ts}
 \end{figure}
 
 \begin{figure}[h]
 \centering
 \includegraphics[width=\columnwidth]{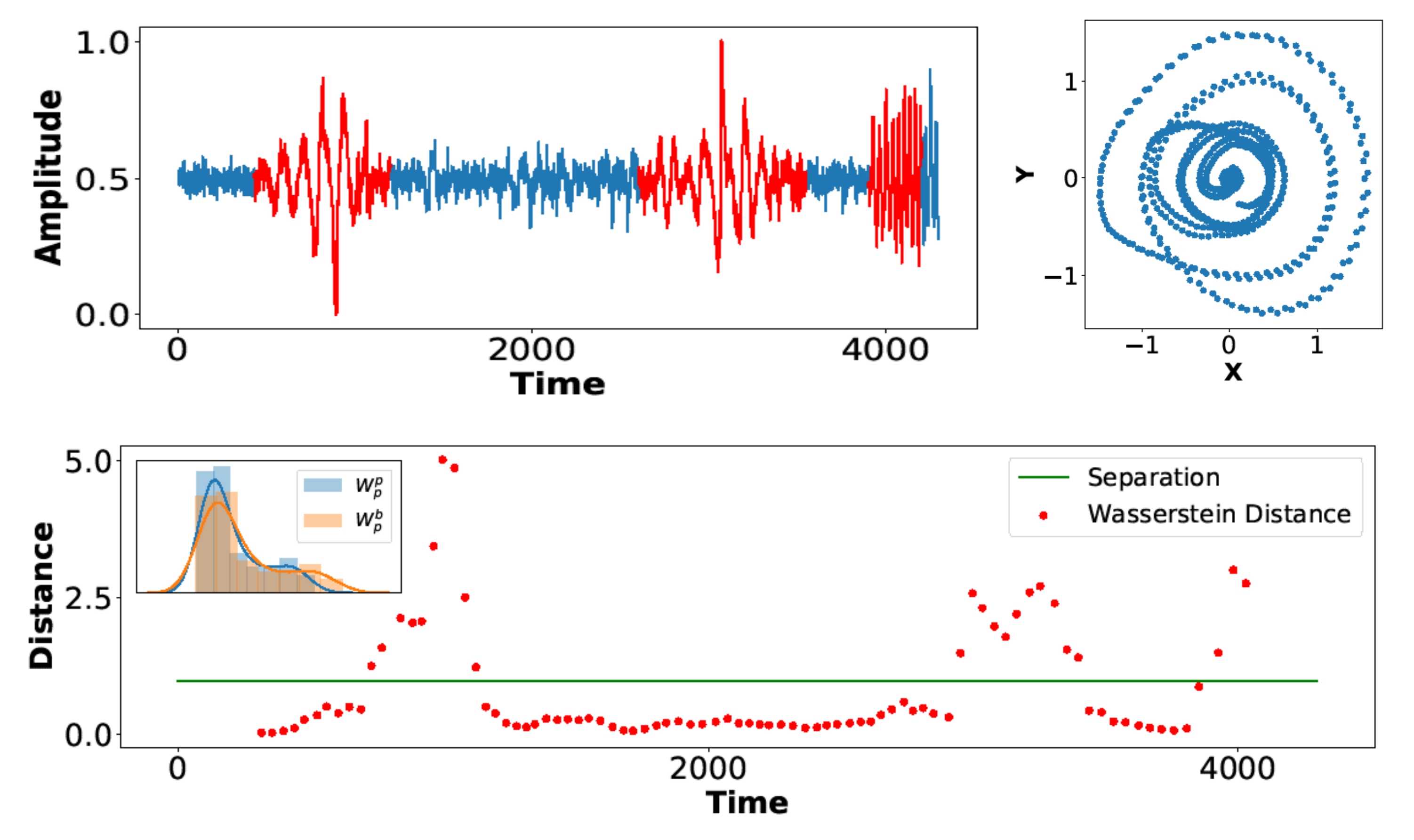}
 \caption{Algorithm's performance on a real Parkinson's disease tremor data. Wasserstein plot and the corresponding point cloud.}
 \label{fig:tremor}
 \end{figure}
 
 \subsection{Other datasets}
 We went beyond ECG, and tested our method on other quasi-periodic physiological signals, yielding the following metrics: neuron spike activity changes\cite{krylov2019RL} were detected with sensitivity of 94.6\% and specificity of 88\%, and a limb tremor data in a patient with intermittent episodes of increased symptoms of parkinsonism\cite{somov2019}(intermittent tremor) -- with sensitivity of 92.3\%, and specificity of 96\%.

 \section{Conclusion}
 We presented a new unsupervised and non-parametric learning algorithm for detection of arrhythmias and of other rhythm anomalies in the raw data of quasi-periodic recordings. The detection relies on optimal transport theory combined with topological analysis and the bootstrap procedure, with the convergence of the bootstrap procedure being proven theoretically. The simple pipeline provides a robust statistical approach to predict abnormal rhythms in an unsupervised manner with high computational efficiency. 
 
 Despite already demonstrating the level of performance of the supervised algorithms, our approach is expected to perform even better if combined with the deep learning methods (similarly to~\cite{Guidi}), especially in a recurrent neural network configuration.  Another line of the future work can entail the extension of the algorithm for the multi-class classification also using the unsupervised bootstrap method on the Wasserstein distances. %The presented algorithm is scalable and universal, allowing to classify quasi-periodic signals beyond electrocardiography.

\bibliographystyle{named}
\bibliography{main}
\end{document}